\documentclass[sigconf, nonacm]{acmart}

\usepackage{enumitem}
\usepackage{tabularx}
\usepackage{multirow}
\usepackage[ruled, linesnumbered]{algorithm2e}
\def\BibTeX{{\rm B\kern-.05em{\sc i\kern-.025em b}\kern-.08em
    T\kern-.1667em\lower.7ex\hbox{E}\kern-.125emX}}
\usepackage{hyperref}
\hypersetup{
    colorlinks,
    linkcolor={red!50!black},
    citecolor={blue!50!black},
    urlcolor={blue!80!black}
}

\begin{document}
\title{GFS: Graph-based Feature Synthesis for Prediction over Relational Databases}

\author{Han Zhang$^{1}$, Quan Gan$^{2}$, David Wipf$^{2}$, Weinan Zhang$^{1}$}
\affiliation{
  \institution{$^{1}$Shanghai Jiao Tong University, $^{2}$AWS Shanghai AI Laboratory}
}
\email{han.harry.zhang@gmail.com, {quagan, daviwipf}@amazon.com, wnzhang@sjtu.edu.cn}
\newcommand{\david}[1]{{\color{red} [David: ``#1'']}}
\newcommand{\han}[1]{{\color{blue} [Han: ``#1'']}}
\newcommand{\quan}[1]{{\color{olive} [Quan: ``#1'']}}
\newcommand{\weinan}[1]{{\color{magenta} [Weinan: ``#1'']}}
\newcommand{\minjie}[1]{\color{yellow} [Minjie: ``#1'']}

\newcommand{\model}{{GFS}}

\newtheorem{remark}{Remark}

\def\header{\vspace{2.5mm} \noindent}

\begin{abstract}
Relational databases are extensively utilized in a variety of modern information system applications, and they always carry valuable data patterns. There are a huge number of data mining or machine learning tasks conducted on relational databases. However, it is worth noting that there are limited machine learning models specifically designed for relational databases, as most models are primarily tailored for single table settings. Consequently, the prevalent approach for training machine learning models on data stored in relational databases involves performing feature engineering to merge the data from multiple tables into a single table and subsequently applying single table models. This approach not only requires significant effort in feature engineering but also destroys the inherent relational structure present in the data.  To address these challenges, we propose a novel framework called Graph-based Feature Synthesis (\model{}). \model{} formulates the relational database as a heterogeneous graph, thereby preserving the relational structure within the data. By leveraging the inductive bias from single table models, \model{} effectively captures the intricate relationships inherent in each table. Additionally, the whole framework eliminates the need for manual feature engineering. In the extensive experiment over four real-world multi-table relational databases, \model{} outperforms previous methods designed for relational databases, demonstrating its superior performance.
\end{abstract}

\maketitle
\pagestyle{plain}
\begingroup
\renewcommand{\thefootnote}{}
\footnote{\noindent \emph{Preprint}, under review}
\addtocounter{footnote}{-1}
\endgroup

\section{Introduction}

Data mining originates from the tasks of mining useful patterns from databases. \emph{Column prediction}, where a model is trained to predict the values in a certain \emph{target column} in a single \emph{target table} of the database, is an important task. A wide range of applications, such as click-through-rate prediction\cite{guo2017deepfm, cheng2016wide, zhou2018deep, guo2022miss}, anomaly detection\cite{han2022adbench, liu2008isolation, zhao2018xgbod, tang2022rethinking}, frequent pattern mining \cite{han2007frequent, agrawal1994fast,han2000mining}, etc., can all be formulated as column prediction.

However, almost all previous data mining works focus the scope of column prediction on a single table. If there are multiple tables and a schema, a common approach is to first merge these tables into a single one, and then perform machine learning methods given the merged single table \cite{park2022end, rendle2010factorization,qu2018product}.
These approaches rely on manually joining these tables into a single table as a feature engineering step. Executing manual feature engineering necessitates a significant amount of effort and substantial domain expertise. An often-cited statistic is that data scientists spend 80\% of their time integrating and curating their data \cite{create2020}. Moreover, it can destroy the inherent relational structure embedded within the data, potentially leading to a loss of valuable information. The need and interest in mining information from various domains of relational databases have been growing, with a significant emphasis on reducing human effort and minimizing information loss.

To this end, we propose a general novel framework called Graph-based Feature Synthesis (\model{}).
\model{} is an embedding update and prediction framework that can plug in any differentiable model designed for single table settings 
as base model or embedding function to finalize the specific model.
By leveraging the inductive bias from single table models, \model{} effectively captures the intricate relationships inherent in each table. \model{} formulates the relational database as a heterogeneous graph, thereby preserving the relational structure within the data. Additionally, the whole framework eliminates the need for manual feature engineering.

Some existing solutions address a similar issue, yet they possess certain shortcomings that can be remedied by the design of the \model{} framework.
Automated methods represent a category of techniques that consolidate multiple tables into a single table using predefined rules \cite{kanter2015deep, lam2017one, arda, liu2022feature, galhotra2023metam}. DFS (Deep Feature Synthesis) \cite{kanter2015deep}, officially implemented by Featuretools \cite{featuretools}, is the most remarkable representative of the offline methods. %
It utilizes depth-first search to aggregate or join columns from other tables to the target table. %
Despite its conciseness, DFS has several inherent problems to be addressed.
(1) DFS only offers rule-based aggregation functions such as \textit{mean}, \textit{max}, and \textit{min} for continuous features, or \textit{mode} for categorical features. Such a limited set of aggregation functions reduces the overall effectiveness of the method and results in the loss of information during the aggregation process.
(2) The search process employed by DFS relies on depth-first search, making the results sensitive to the order in which the children tables are traversed for a given parent table.  As we show in Section~\ref{sec:property}, it may ignore some join paths that could be crucial for the downstream task, even if the search depth is increased.
(3) DFS aggregates every column of the children table using several aggregation functions and appends it to the parent table. Consequently, the number of columns appended to the parent table is multiplied by the number of aggregation functions used. This results in an exponential increase in the number of columns when the depth of the search is deep or when the children table contains a large number of columns.

Recently, methods that try to convert the database into a graph and apply Graph Neural Networks (GNNs) are also developed. When making prediction of a single data instance, GNNs are able to use information of other related data instances.  RDB2Graph (Relational Database to Graph) \cite{cvitkovic2020supervised} is one of the most representative approaches, which constructs a graph by %
treating each row as a node, each foreign key reference as an edge, and running a GNN on top of the raw input features.  %
However, RDB2Graph has several drawbacks:
(1) RDB2Graph solely relies on an MLP to form the initial node embeddings for each row, which limits its ability to explore the interaction patterns among columns effectively. In contrast, there exist a number of single table models that excel in capturing such column interactions \cite{rendle2010factorization,guo2017deepfm,qu2018product}.
(2) The GNN model utilized by RDB2Graph is not specifically designed for tabular data, and its simplicity may result in suboptimal performance when handling relational databases.
(3) GNN models are notoriously susceptible to an over-smoothing problem when the number of layers becomes too large. However, in complex relational databases, capturing information from all tables necessitates constructing a substantial number of layers, exacerbating this problem.

In this paper, we propose a novel framework called Graph-based Feature Synthesis (\model{}) which offers several advantages over existing approaches.  For example, in comparison to DFS, \model{}  introduces several key improvements:
(1) Learnable aggregation functions: Instead of being limited to rule-based aggregation functions, \model{} enhances the capability by incorporating parameterized learnable functions, which provide increased flexibility and power in capturing complex relationships.
(2) Removing sensitivity to traversal order: \model{} is not affected by the order in which tables are traversed, which ensures more consistent and reliable results regardless of the sequence in which traversal occurs.  In contrast, we prove that DFS is \textit{not} similarly invariant.
(3) Complete coverage of all join paths: Unlike DFS, \model{} is capable of covering all join paths within a specified number of hops, which means that no important traversal path will be missed, enabling a comprehensive modeling of the relational structure. 
(4) Controlled column growth: Rather than appending every column from child tables, \model{} simply appends the row embedding of each row to the parent table. This significantly alleviates the issue of exponential column growth, yielding a more manageable and efficient representation.

Turning to RDB2Graph, \model{} also exhibits several notable improvements:
(1) Enhanced row embeddings using powerful single table models: instead of relying solely on MLPs, \model{} incorporates well-designed and powerful single table models to generate row embeddings. %
(2) Improved prediction using advanced single table models: we utilize powerful single table models for the final prediction step in \model{}. %
(3) Input residual connections to mitigate over-smoothing \cite{alon2020oversquashing,chen2020oversmoothing} to address the over-smoothing problem commonly associated with GNNs, we incorporate residual connections between the raw node features and every layer. %

To sum up, the main contributions of this paper are threefold:
\begin{itemize}[leftmargin=*]
    \item After illustrating the need for models specifically designed to handle \emph{column prediction tasks} within relational databases, we propose the novel and generic \model{} framework. Based on modular components, \model{} can incorporate any
    differentiable model designed for single table settings as an embedding function and/or prediction head.  In so doing we can equip \model{} with the strong inductive biases of existing single table models, and benefit from future advances in this field.
    \item Relative to alternative representative paradigms in the mold of DFS and RDB2Graph, \model{} offers targeted improvements, such as invariance to traversal order, greater expressiveness, and oversmoothing mitigation.
    \item Experiments conducted on four real-world relational datasets demonstrate that \model{} outperforms existing methods, while ablation studies validate the importance of each \model{} component.
\end{itemize}

The remainder of the paper is organized as follows. Section \ref{sec:related_works} first introduces more related work, and Section \ref{sec:problem_definition} then introduces the formal task definition along with background terminology and notation.  Next, Section \ref{sec:model} presents our detailed \model{} framework, followed by Section \ref{sec:property} which provides theoretical comparisons between \model{} and two key baseline alternatives, DFS and RDB2Graph. Finally, in-depth experiments and conclusions are presented in Sections \ref{sec:experiments} and \ref{sec:conclusion} respectively.

\section{Related Works}
\label{sec:related_works}
There is a significant demand in the industry for prediction tasks on relational databases \cite{netz2000integration, park2022end, aggarwal2012analysis, shahbaz2010data, fernandez2018aurum}.
However, it is true that there has been relatively limited research specifically focusing on prediction in the context of relational databases.

\header
\textbf{Feature Engineering.}
In industry, the prevailing approach for prediction tasks involves manual feature engineering on the relational database, followed by the application of a single table model on the resulting single table\cite{park2022end, rendle2010factorization,qu2018product}.
This method requires significant effort in data extraction and feature engineering. Additionally, it heavily relies on the specific task and dataset at hand, necessitating a deep understanding of domain knowledge. However, these efforts are costly and have the drawback of potentially undermining the crucial relational structure inherent in the data.

\header
\textbf{Automatic Feature Engineering.}
In the field of automating feature engineering in relational databases, several notable works have been introduced, including DFS \cite{kanter2015deep}, OneBM (One Button Machine) \cite{lam2017one}, R2N \cite{lam2018neural}, ARDA \cite{arda}, METAM \cite{galhotra2023metam} and AutoFeature \cite{liu2022feature}. DFS \cite{kanter2015deep}, being the most widely used and open-source method, employs depth-first search to aggregate or join columns from other tables to the target table using rule-based aggregation functions. However, as mentioned in the introduction, DFS has certain limitations that we have addressed. 
OneBM shares similarities with DFS as it also utilizes rule-based aggregation functions. However, it improves upon DFS by enumerating the traversal paths, thereby reducing the variance issue to some extent. On the other hand, R2N extends the rule-based aggregation function to LSTM \cite{hochreiter1997long}. ARDA is another system that automates the data augmentation workflow and uses feature selection. AutoFeature arguments feature from candidate tables to the target table following an exploration-exploitation strategy with a reinforcement learning-based framework. METAM is a novel goal-oriented framework that queries the downstream task with a candidate dataset, forming a feedback loop that automatically steers the discovery and augmentation process.

\header
\textbf{Graph-based Methods.}
In the introduction, we analyze RDB2Graph, which represents a preliminary attempt at utilizing GNNs for prediction tasks in relational databases. GNNs are neural networks specifically designed for data structured as graphs. Typically, GNN architectures involve three main steps: node embedding initialization, message passing, and readout.
In the node embedding initialization step, each node is assigned an initial node embedding vector. These initial embeddings capture the initial representations of the nodes.
The message-passing step involves nodes utilizing their own node embedding vectors, as well as those from their neighboring nodes, to update their own node embeddings. This allows information to be propagated and aggregated across the graph.
The readout part employs the final node embeddings to perform predictions, with the specific approach depending on the downstream task.
In the context of a relational database, it is natural and intuitive to consider the database as a heterogeneous graph. This perspective allows for the application of GNN architectures designed specifically for heterogeneous graphs. Several popular examples of such architectures are the RGCN (Relational Graph Convolution Network) \cite{schlichtkrull2018modeling}, attention-based RGAT (Relational Graph Attention Network) \cite{busbridge2019relational}, transformer-based HGT (Heterogeneous Graph Transformer) \cite{hu2020heterogeneous}. 
Furthermore, Cvetkov et al. \cite{cvetkov2023relational} propose a knowledge graph method that utilizes tables alongside the target table to enhance the features of the target table. However, the results presented in that paper demonstrate that DFS performs best in most datasets, remaining a highly effective baseline. Additionally, ATJ-Net\cite{bai2021atj} constructs hypergraphs to fuse related tables and use architecture search to prone the model.

\header
\textbf{Single table models.}
While models specifically designed for relational databases are comparatively rare, numerous models have been proposed for single-table settings. Among these, tree-based models such as Gradient Boosting Decision Trees (GBDT) \cite{friedman2001greedy} and Factorization Machines (FM) \cite{rendle2010factorization} have gained significant attention. Variants of FM, including DeepFM \cite{guo2017deepfm} and Wide \& Deep \cite{cheng2016wide}, are particularly popular in industrial applications and other feature interaction model such as ARM-Net\cite{cai2021arm}, DCN-V2\cite{wang2021dcn}. Graph-based models such as EmbDI\cite{create2020}, TabularNet\cite{du2021tabularnet}. Furthermore, transformer-based single table models have also recently emerged, such as FT-Transformer \cite{gorishniy2021revisiting}, TabTransformer \cite{huang2020tabtransformer}, TURL \cite{deng2022turl} and RCI \cite{glass2021capturing}.

\begin{figure}
    \centering
    \includegraphics[width=\linewidth]{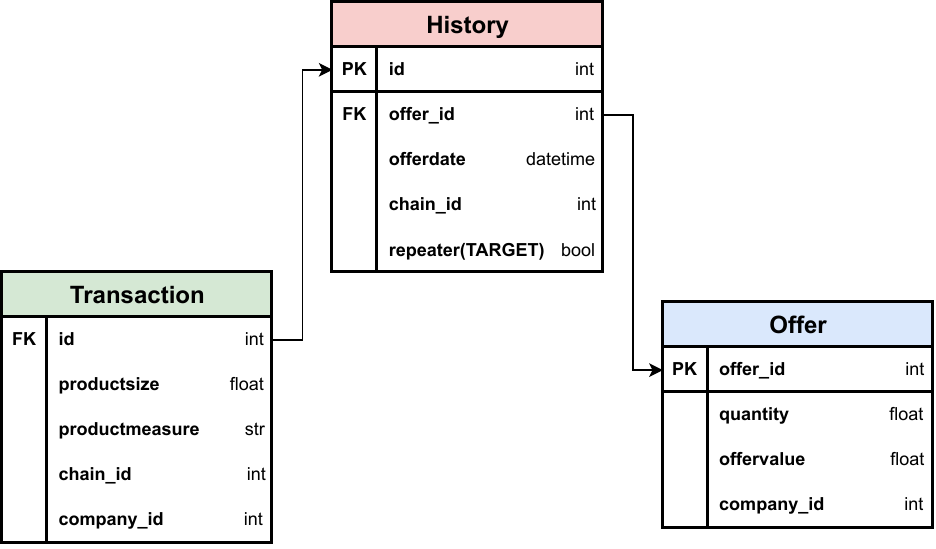}
    \caption{Example schema from AVS dataset. PK stands for primary key, FK for foreign key, rest are other columns.}
    \label{fig:AVS}
\end{figure}

\section{Problem Definition}\label{sec:problem_definition}

\begin{table*}
    \caption{The mapping from a relational database to a graph.}
    \centering
    \begin{tabularx}{0.9\textwidth}{p{8cm}l}
        \toprule
        \textbf{Relational Database} & \textbf{Graph} \\
        \midrule
        Row & Node \\
        Foreign key reference from $X^{A}_{u,i}$ to $X^{B}_{v,j}$ & Directed edge from node $u$ of type $A$ to node $v$ of type $B$ \\
        Table & Node type \\
        Foreign key column & Edge type \\
        Non-Foreign key or Primary key column & Node raw feature \\
        \bottomrule
    \end{tabularx}
    \label{tab:rel2graph}
\end{table*}

Let $D=\{X^1, X^2, \ldots , X^N\}$ denote a relational database, where each $X^n$ represents the $n$-th consituent table. An example is shown in Fig. \ref{fig:AVS}.
The $i$-th row and $j$-th column of table $X^n$ are denoted as $X^{n}_{i,:}$ and $X^{n}_{:, j}$ respectively. In this way, $X^{n}_{i,j}$ represents the entry value of row $i$ and column $j$ of table $X^n$. We use $C_n$ and $R_n$ to denote the number of columns and rows in $X^n$, respectively. Proceeding further, two tables relate to each other in one or two ways: \textit{forward} or \textit{backward}, similar to \cite{kanter2015deep}, as described next.

\header
\textbf{Forward}: Suppose that there is a foreign key reference from $X^{A}_{:,i}$ to $X^{B}_{:,j}$ and $X^{B}_{:,j}$ is the primary key in $X^B$. Then a \textit{forward} relationship is from one row $X^{A}_{m, :}$ to another row $X^{B}_{l, :}$ since $X^{A}_{m, i}$ has foreign key reference to $X^{B}_{l, j}$, i.e., $X^{A}_{m, i} = X^{B}_{l, j}$  
We can state that $X^A_{m,:}$ has forward relationship to $X^B_{l,:}$ or table $X^{A}$ has a forward relationship to $X^{B}$, and actually each foreign column in $X^{A}$ represents one type of forward relationship. In the example shown in Fig.~\ref{fig:AVS}, the table \textit{History} has a forward relationship to table \textit{Offer}, where each row in \textit{History} represents a customer, and one customer only refers to one offer.

\header
\textbf{Backward}: A \textit{backward} case refers to the relationship from one row $X^{B}_{l, :}$ to all the rows $\{X^{A}_{m_1, :}, X^{A}_{m_2, :}, \ldots ,X^{A}_{m_n, :}\}$ that have forward relationship to $X^{B}_{l, :}$ due to the foreign key reference from $X^{A}_{:,i}$ to $X^{B}_{:,j}$. In the same example as above, the table \textit{Offer} has a backward relationship to \textit{History}, and many customers can have the same offer.

With all the necessary terminology and notations now clearly defined, we proceed to introduce the specific task that this paper addresses. A wide range of data mining problems over a relational database $D$ can be formulated as a \emph{column prediction} task, which will be our focus. We formulate such problems as follows: 
\begin{definition}[Column Prediction Task on Relational Database] \label{def:basic_problem}
    Given a relational database $D$ and all the relationships between tables, predict the values in a target column $X^{T}_{:, \text{target}}$ of a \emph{target table} $X^{T} \in D$ of interest using relevant information available in the database.%
\end{definition}

\section{The \model{} Framework}\label{sec:model}

\begin{figure*}
    \centering
    \includegraphics[width=\linewidth]{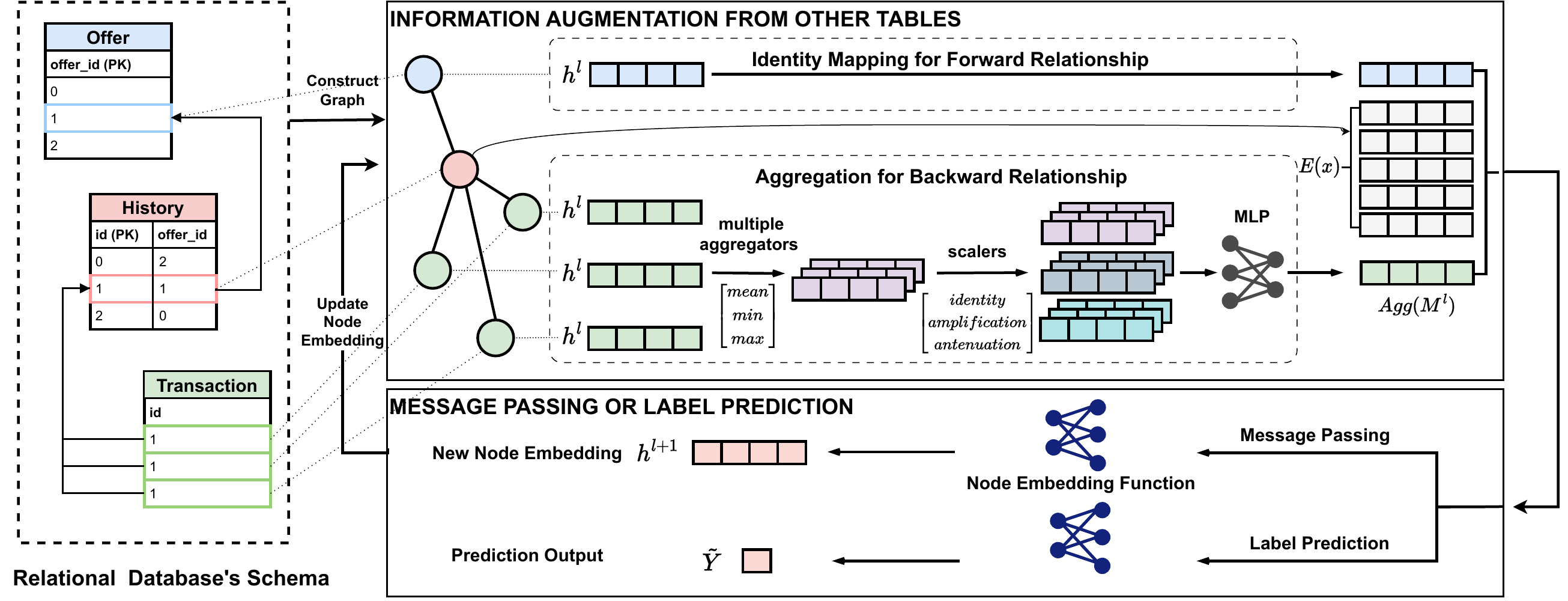}
    \caption{Overview of the \model{} framework. We use the node highlighted in red as an example to demonstrate how the node in the graph updates its node embedding and how predictions are made if the node is the target node. Some notations are abbreviated and some columns in tables are not shown for simplicity.}
    \label{fig:GFS}
\end{figure*}
We now present the technical details of the proposed \model{} framework designed to tackle the column prediction problem specified by Definition \ref{def:basic_problem}.  To this end, we first describe the process of converting a relational database to a graph, the nodes of which have learnable embeddings.  After discussing the initialization of these embeddings, we then proceed to the core training and inference steps, namely, message passing and label prediction.  The overview of \model{} is shown in Fig.~\ref{fig:GFS}, with the associated pseudo-code in Algorithm~\ref{algo:GFS} below.

\subsection{Interpreting a Relational Database as a Graph}

A relational database can be interpreted as a heterogeneous directed graph, with the correspondences depicted in Table \ref{tab:rel2graph}.  We consider each row of a table as a separate node in the entire graph, and all rows within a table are treated as nodes of the same type. If there is a foreign key reference from $X^{A}_{u,i}$ to $X^{B}_{v,j}$, it implies the existence of a directed edge from node $u$ of type $A$ to node $v$ of type $B$. 
Additionally, foreign key references within the same foreign key column are considered as edges of the same type. For instance, in the schema example of Fig.~\ref{fig:AVS}, all foreign key references from the table \textit{Transaction} to the table \textit{History} will be categorized as the same edge type.

After the heterogeneous directed graph is constructed, we then add reverse edges for each edge in order to let the whole graph become a heterogeneous undirected graph which is more suitable for aggregating the information from other tables to the target table.

\header
\textbf{Node (Row) Raw Features.} The values of a row in columns other than primary and foreign keys are regarded as node(row) raw features. Examples include the values in the \textit{quantity, offervalue, company\_id} columns from table \textit{Offer} in Fig. \ref{fig:AVS}. We denote the raw node features associated with example row $X_{i,:}^n$ as 
\begin{equation}
    x=[X_{i,1}^n, X_{i,2}^n, \ldots, X_{i,C_n}^n]~.
\end{equation}

\header
\textbf{Node (Row) Embedding.} The low-dimensional vector which encodes node-wise information is defined as a node embedding. We denote the node embedding associated with $X_{i,:}^n$ as $h(i, n)\in \mathbb{R}^d$, 
where $n$ references the table $X^n$ and $i$ represents the corresponding intra-table row.

\subsection{Embedding Initialization}
We allow for three types of data in our framework, namely, \textit{categorical}, \textit{continuous}, and \textit{date}. To facilitate follow-up calculations, raw node features of each type are first transformed into dense real-valued embedding vectors.

\header
\textbf{Categorical Values.} 
Categorical values $c_i\in \mathbb{Z}$ are embedded into $e_i\in \mathbb{R}^{d}$ via a column-wise look-up table, mapping each categorical value to a vector. And different columns that share the same semantic meaning will use the same look-up table. In this way, in Fig.~\ref{fig:AVS} the same categorical value $c_i$ from the \textit{chain\_id} column in table \textit{Transaction} and from the table \textit{History} will share the same embedding vector $e_i$.

\header
\textbf{Continuous Values.} 
For a continuous value $c_i\in \mathbb{R}$ within a given column, we first normalize to the interval $[0, 1]$.  We then apply the transform 
\begin{equation}
    e_i = \text{ReLU}(c_i\times w_f+b_f) \in \mathbb{R}^{d}~,
\end{equation}
where $w_f, b_f\in \mathbb{R}^{d}$ are the weight and bias of a linear transformation layer in column $f$, and all the continuous values in one column share the same weights and biases.

\header
\textbf{Dates.}
We encode one date column into four columns, namely \textit{year, month, day of month, day of week}, and treat the values in each column as categorical.

\header
In aggregate across feature types, we convert the example row $X^{n}_{i,:}$'s raw features $x$ into dense real-valued embedding vectors as 
\begin{equation}
    E_{\theta}(x)\in \mathbb{R}^{C_n\times d}~,
\end{equation}
where $\theta$ represents the learnable parameters of the embedding process.

\subsection{Message Passing}\label{sec:message_passing}
We introduce message passing steps by taking the node corresponding to row $i$ in table $X^n$ as an example and use $x$ to denote the row raw feature as above. We define $F=\{X^{f_1}, X^{f_2},\ldots, X^{f_{|F|}}\}$ for all the tables $X^{n}$ has forward relationship to. Analogously we define $B=\{X^{b_1}, X^{b_2},\ldots, X^{b_{|B|}}\}$ for backward relationships. Note that both sets can have repetitive elements since $X^n$ can have multiple types of forward or backward relationships with one table.

\header
\textbf{Node (Row) Embedding Function.}
After transforming the node raw feature $x$ to the dense real-valued embedding vectors and gathering information from other node(will explain later), we adopt the node embedding function
\begin{equation}
    N_{\phi}: \mathbb{R}^{(C_n+|F|+|B|)\times d}\rightarrow \mathbb{R}^{d}
\end{equation}
to obtain the node embedding, where 
$\phi$ represents trainable parameters and the input dimension will be 
 further explained in the Message Passing Function part. The node embedding function can be similar to any single table model such as an $\text{MLP}$, $\text{FM}$~\cite{rendle2010factorization}, or $\text{FT-Transformer}$~\cite{gorishniy2021revisiting}. Additionally, nodes of the same type (rows in the same table) will share the same node (row) embedding function.

\header
\textbf{Aggregation Function. }
For the table that $X^{n}$ has a backward relationship to, we need to aggregate the node embedding of the nodes which relate to node $X^{n}_{i,:}$. This function takes a set of node embedding $M$, and the output is a vector of dimension $d$:
\begin{equation}
    \text{Agg}_{\psi}: \mathbb{M}\rightarrow\mathbb{R}^{d}~,
\end{equation}
where $\psi$ is the learnable parameter of the aggregation function $\text{Agg}_{\psi}$, $\mathbb{M}$ is the set satisfies that any set of node embedding $M\in\mathbb{M}$. Note that the aggregation function is independent of the number of node embedding in $M$ and the same aggregation function will be used for the same edge type (foreign key column).

We use an aggregation method that is similar to PNA \cite{corso2020principal}. Specifically, we calculate the \textit{mean, max, min} of the overall embedding in $M$ and multiply these three vectors three times with three different scalars: (1) \textit{identity}: $S=1$ (2) \textit{amplification}: $S\propto \log(degree)$ (3) \textit{attenuation}: $S\propto \log(degree)^{-1}$. Subsequently, we obtain nine vectors and ultimately map all these vectors into a single vector. The diagram of this aggregation process is in Fig.~\ref{fig:GFS}.

\header
\textbf{Message Passing Function.}
We will explain how the node in the whole graph gets the information from other nodes here.

First, we initialize every embedding $h$ for all the nodes as $\Vec{0}\in \mathbb{R}^d$, then we will update embedding $h$ for each node as \eqref{eq:message} and \eqref{eq:update}, the notation here is for the example node corresponding to row $i$ in table $X^n$:
\begin{align}
\text{Mes}(i, n, l):=& N_\phi\big(\text{concat}[E_\theta(x), \tilde{h}^{l}(i, f_1), \ldots , \tilde{h}^{l}(i, f_{|F|}),\label{eq:message}\\
&\text{Agg}_{\psi_1}(M^{l}(i, b_1)), \ldots,\text{Agg}_{\psi_{|B|}}(M^{l}(i, b_{|B|}))]\big)~\nonumber\\
h^{l+1}(i, n) =& \text{Mes}(i, n, l)~. \label{eq:update}
\end{align}

Repeating the above step $k$ times, we can get at most $(k-1)$-hops information. And we also define the search depth $K$ as the repeating times of above step.
\begin{itemize}[leftmargin=*]
\item We define $\tilde{h}^{l}(i, t)\triangleq h^l(i_t, t)$, and $i_t$ stands for the row number in table $X^t$ that $X^{n}_i$ has forward relationship to. $\tilde{h}^{l}(i, f_1), \ldots, \tilde{h}^{l}(i, f_{|F|})$ are $l$-step results node embedding of the rows in $X^{f_1},\ldots,X^{f_{|F|}}$ that $X^{n}_i$ has forward relationship to. 
\item $M^{l}(i, b_1),\ldots,M^{l}(i, b_{|B|})$ are $l$-step results node embedding sets of the rows in $X^{b_1},\ldots,X^{b_{|B|}}$that $X^{n}_i$ has backward relationship to. For example, $X^{b_1}_{j_1, :}, X^{b_1}_{j_2, :},\ldots,X^{b_1}_{j_m, :}$ are all the rows in $X^{b_1}$ that $X^{n}_i$ has backward relationship to, so $$M^{l}(i, b_1)=\{h^{l}(j_1, b_1), h^{l}(j_2, b_1), \ldots , h^{l}(j_m, b_1)\}$$
\item $N_\phi$ represents the node embedding function for each node, and you can also regard the node embedding function as the row embedding function for each row. Note that the dimension of $E_\theta(x)$ is $C_n\times d$, dimension of $[\tilde{h}^{l}(i, f_1), \ldots, \tilde{h}^{l}(i, f_{|F|})]$ is $|F|\times d$, dimension of $[\text{Agg}_{\psi_1}(M^{l}(i, b_1)), \ldots,\text{Agg}_{\psi_{|B|}}(M^{l}(i, b_{|B|}))]$ is $|B|\times d$ so after concatenating all these vectors, the input dimension of this function is $(C_n+|F|+|B|)\times d$.
\end{itemize}

\subsection{Label Prediction}\label{sec:label_prediction}
\header
\textbf{Base Model.}
Base Model is the model that \model{} uses for the final label prediction. Since \model{} has skip connection from the Node Raw Feature, we can use any model in the single table setting as the base model. Specifically, a base model $\pi_{\omega}$ can be written as
\begin{equation}
    \pi_{\omega}: \mathbb{R}^{(C_n+|F|+|B|)\times d} \rightarrow \mathbb{R}^{p}~,
\end{equation}
where $\omega$ is the learnable parameter of the base model, and $p$ is the dimension of the output vector.

We use the target table's node to do prediction as \eqref{eq:prediction} and \eqref{output}. The notation here is for the example node corresponding to row $i$ in table $X^n$. 
\begin{align}
\text{Pred}(i, n, l) :=& \pi_{\omega}\big(\text{concat}[E_\theta(x),\tilde{h}^{l}(i, f_1), \ldots , \tilde{h}^{l}(i, f_{|F|}),\label{eq:prediction}\\ 
&\text{Agg}_{\psi_1}(M^{l}(i, b_1)), \ldots, \text{Agg}_{\psi_{|B|}}(M^{l}(i, b_{|B|}))]\big)\nonumber \\
\tilde{Y}^n_i =& \text{Pred}(i, n, l)~. \label{output}
\end{align}
where $\tilde{Y}^n_i$ represents the output prediction vector of \model{} by the example node.

\begin{algorithm}
\caption{Graph-based Feature Synthesis}\label{algo:GFS}
\SetKwInOut{Ip}{Input}
\SetKwInOut{Op}{Output}
\Ip{All tables $X^{1:N}$; search depth $K$}
\Op{Prediction $\tilde{Y}$}
\tcp{message passing}
$h^0(i,n)\leftarrow\Vec{0}\in \mathbb{R}^d\,\,\forall n\in[1, N], i\in[0, R_n]$\\
\For{$l=1\ldots K$ 
}{
    \For{$n=1\ldots N$}{
        \For{$i=1\ldots R_n$}{
            $h^l(i, n)\leftarrow \text{Mes}(i, n, l-1)$  
        }
    }
}
\tcp{label prediction}
\For{$i=1\ldots R_{T}$}{
    $\tilde{Y}^{T}_i=\text{Pred}(i, T, K)$
}
\end{algorithm}

\subsection{\model{} Time Complexity}
In the actual computation, we will use the $K$-hop neighbor subgraph induced by the target node with the search depth $K$ to compute the prediction on the target node. 

Suppose the subgraph $\mathcal{G}$ has $N$ nodes, the maximum degree of the nodes in $\mathcal{G}$ is $\Delta$, the dimension of node embedding is $d$ as defined before, and the maximum number of columns among each table is $c$.
Considering the message passing step for each node, the aggregators(mean, max, min) will cost $O(\Delta d)$, and the MLP transformation will cost $O(d^2)$. Assume that the computation cost of node embedding function is $O(P)$ 
. Then the time complexity of updating one node's embedding is $O(P+d^2+\Delta d)$ and we need to update $K$ times.

In the label prediction step, the computational cost of aggregation is $O(\Delta d+d^2)$, and assume the computation cost of base model is also $O(P)$.Then the time complexity of doing the label prediction on the target node is $O(P+d^2+\Delta d)$.
Finally, the total time complexity is $O(KN(P+d^2+\Delta d))$ for the target node. 

To specify, if the node embedding function and base model is FT-transformer, $O(P)=O(cd^2+c^d)$, the total time complexity is $O(KN(cd^2+c^2d+\Delta d))$. And actually the order of magnitude of the computation cost for FT-transformer is representative of higher than many single table models.

If we delve more into the time complexity, suppose the original whole graph $\tilde{\mathcal{G}}$ has $\tilde{N}$ nodes, and the maximum degree of the nodes in $\tilde{\mathcal{G}}$ is $\tilde{\Delta}$, we can have at most $\tilde{\Delta}^K$ nodes in subgraph $\mathcal{G}$, so the total complexity is $min(O(K\tilde{\Delta}^K(P+d^2+\tilde{\Delta}d)), O(K\tilde{N}(P+d^2+\tilde{\Delta}d)))$, and it is $min(O(K\tilde{\Delta}^K(cd^2+c^2d+\tilde{\Delta}d)), O(K\tilde{N}(cd^2+c^2d+\tilde{\Delta}d)))$,

\section{Comparative Analysis}\label{sec:property}
As mentioned previously, two strong candidates for column prediction tasks on relational databases are DFS-based and RDB2Graph-based solutions.  To better contextualized GFS relative to these alternatives, this section presents the following comparative analysis.  First, we prove that DFS is not invariant to traversal order, and hence may produce different results leading to undesirable instabilities.  In contrast, GFS outputs are uniquely determined with no analogous ambiguity.  We then prove that GFS generalizes DFS in the sense that there will exist GFS parameterizations that can match any DFS output.  And lastly, we closely examine key factors that differentiate GFS and RDB2Graph, including key design choices unique to the former.

\subsection{DFS sensitivity to traversal order}

We compare the capability of \model{} versus DFS by first noting an important limitation of DFS:

\begin{theorem}\label{thm:DFS variant}
    Given a fixed set of input tables, one of which serves as the target, the output of DFS need not be invariant to the traversal order.
\end{theorem}

\begin{algorithm}
\SetKwFunction{MF}{Make\_Features}
\SetKwFunction{Ba}{Backward}
\SetKwFunction{Fo}{Forward}
\SetKwFunction{Rf}{Rfeat}
\SetKwFunction{Df}{Dfeat}
\SetKwFunction{Ef}{Efeat}
\SetKwProg{Fn}{Function}{}{end}
\caption{Deep Feature Synthesis}\label{alg:DFS}
\SetKwInOut{Ip}{Input}
\SetKwInOut{Op}{Output}
\Ip{Target table $X^i$, set of visited tables $X_V$, search depth $K$}
\Op{Augmented target table $X^i$}
    \Fn{\MF{$X^i$, $X_V$, $K$}}{
    $X_V=X_V\cup X^i$\\
    $X_B=$\Ba{$X^i$}\\
    $X_F=$\Fo{$X^i$}\\
    \For{$X^j\in X_B$}{
        \If{$X^j$ not in $X_V$}{
            $X^j=$\MF{$X^j$, $X_V$, $K-1$}
        }
        $F^i=F^i\cup$ \Rf($X^i$, $X^j$)
    }
    \For{$X^j\in X_F$}{
        \If{$X^j$ not in $X_V$}{
            $X^j=$\MF{$X^j$, $X_V$, $K-1$}
        }
        $F^i=F^i\cup$ \Df($X^i$, $X^j$)
    }
    \Return{$X^i=F^i\cup$ \Ef{$X^i$}}
    }
\end{algorithm}

\begin{proof}
    We present the pseudocode of DFS in Algorithm~\ref{alg:DFS}.  The core of DFS is the recursion of the \texttt{Make\_Features} function, which augments the target table $X^i$ given all the tables $X^1, \ldots, X^n$.  DFS traverses the tables along both backward relationships and forward relationships, denoted by \texttt{Backward} and \texttt{Forward}.  During traversal, it either aggregates the information from tables with backward relationships via the function \texttt{Rfeat} (e.g. SUM), or directly appends the information from tables with forward relationships via the function \texttt{Dfeat}.  We refer to \cite{kanter2015deep} for details of how DFS works.

    From the pseudocode we can see that DFS traverses the undirected schema graph of the tables in a depth-first order.  Therefore, when there is a loop in the underlying undirected schema graph, we can expect that the output of DFS will be different.  We therefore present a counterexample and demonstrate how DFS would process it.

    \begin{figure}
        \centering
        \includegraphics[width=0.8\linewidth]{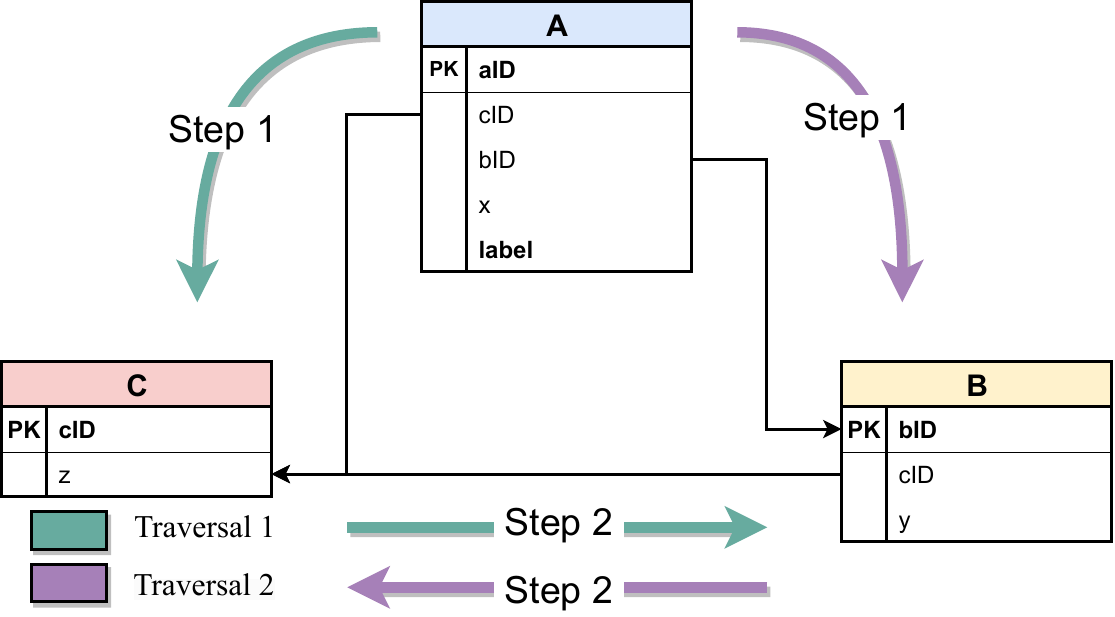}
        \caption{Synthetic Dataset 1's schema (Schema 1) and the traversal path of two traversal types.}
        \label{fig:syn-schema-1}
    \end{figure}

    \header
    As shown in Figure~\ref{fig:syn-schema-1}, we create three tables $A$, $B$, and $C$, each with a single feature column named $x$, $y$ and $z$ respectively.  We set Table $A$ as the target table, assigning it with an additional $label$ column.
    Table A has a foreign key reference to both Table B and C, denoted as $A.bID$ and $A.cID$ \footnote{For the sake of conciseness, we use dot notations to refer to a column or a cell in a table, e.g. $A.x$ refers to the column $x$ in table $A$, and $a.x$ refers to the cell value of row $a$ and column $x$.}, while Table B itself has a foreign key reference to C, denoted as $B.cID$.  Given a depth of 2, there are four possible ways to augment any row $a$ in Table $A$ each corresponding to a join path of length 2 or fewer:
    \begin{enumerate}[leftmargin=*]
        \item Find the row $c$ in $C$ such that $a.cID = c.cID$ and augment $a$ with $c.z$.  This follows the join path $C \rightarrow A$.
        \item Find the row $b$ in $B$ such that $a.bID = b.bID$ and augment $a$ with $b.y$.  This follows the join path $B \rightarrow A$.
        \item Find the row $b$ in $B$ such that $a.bID = b.bID$, then find the row $c$ in $C$ such that $b.cID = c.cID$ and augment $a$ with $c.z$.  This follows the join path $C \rightarrow B \rightarrow A$.
        \item Find the row $c$ in $C$ such that $a.cID = c.cID$, then find all row $b$ in $B$ such that $b.cID = c.cID$, and augment $a$ with the aggregation of all $b.y$.  This follows the join path $B \rightarrow C \rightarrow A$.
    \end{enumerate}
    DFS can, however,  never capture all four of them because of the depth-first nature.  For the schema above, there are two ways for DFS to traverse and perform feature augmentations, which are also shown in Figure~\ref{fig:syn-schema-1}.  In Traversal 1, DFS handles the join path $B \rightarrow C \rightarrow A$ first, so Table $B$ will already be visited when DFS attempts to traverse further down from $A \rightarrow B$, therefore ignoring the possible join path $C \rightarrow B \rightarrow A$.  The same goes for Traversal 2 where join path $B \rightarrow C \rightarrow A$ is ignored.  Therefore, DFS can at most capture three of the four join paths, hence feature augmentations, given the algorithm above.  It can also be seen that the explored join path differs with traversal.
\end{proof}

\begin{remark}
    \model{} does not involve any operations analogous to traversal order, and hence its output is uniquely determined without ambiguity. Moreover,  \model{} can capture all the join paths due to full neighbor sampling.
\end{remark}

We note that the counterexample schema constructed in Figure~\ref{fig:syn-schema-1} may appear in real-world applications.  For example, some music streaming platforms like Soundcloud allow users to upload music, so it could have a user table(Table C), a music table(Table B), and an interaction table(Table A) referencing both the user and the music tables.  The music table there would have a foreign key reference to the user table as well, and so traversal order could subsequently impact DFS outputs.

We further empirically verify the impact of DFS's sensitivity as well as \model{}'s immunity in Section~\ref{sec:synthetic-experiments}, where we construct synthetic datasets based on the schema shown in Figure~\ref{fig:syn-schema-1}, as well as a similarly-structured schema shown in Figure~\ref{fig:syn-schema-2}.

\subsection{\model{} generalizes DFS}

This section will show that \model{} generalizes DFS, meaning that the output of a certain parametrization of \model{} will be a superset of DFS's output.  The intuition is to show that \model{} generalizes DFS under the most simple case of two related tables, where one table has a foreign key reference to the other.  Generalization on more complicated cases can then be established by induction. We will use the notation defined in Section \ref{sec:model}

We denote $\hat{X}^{n,l}_{i,:}$ be the $i$-th row of table $X^n$ augmented by depth-$l$ DFS.  Moreover, we relax the definition of the functions  defined in Section~\ref{sec:model} (e.g. $E_\theta$, $Agg_\psi$, etc.), such that their input and output dimensions can be arbitrary rather than a fixed $d$.

We first show that given two tables $X_1$ and $X_2$, if $X_1$ has a forward relationship to $X_2$, then \model{} generalizes DFS.
\begin{definition}[Subset and superset vectors]
    We define that $A=[a_1, \ldots, a_m]$ is a \emph{subset vector} of $B=[b_1, \ldots, b_n]$ if there exists a one-to-one mapping $\mu: \{1, \cdots, m\} \mapsto \{1, \cdots, n\}$, such that for all $i \in \{1, \cdots, m\}$, $a_i = b_{\mu(i)}$.  We denote it as $A \subseteq^v B$.  Conversely, we say that $B$ is a superset vector of $A$ and denote it as $B \supseteq^v A$.
\end{definition}

\begin{lemma}
    \label{lem:forward}
    Suppose we have two tables $X^1$ and $X^2$ with $C_1$ and $C_2$ columns respectively, and $X^1$ has a forward relationship to $X^2$.  Assume that the columns other than foreign keys and primary keys only have continuous values. There exists a two-layer \model{} parametrization such that for any $i$, $h^{2}(i, 1) \supseteq^v \hat{X}^{1,1}_{i,:}=[\texttt{Make\_Features}(X^1, \varnothing, 1)]_{i,:}$, meaning that the second-layer row embedding from \model{} on $X^1$ is a superset vector of the output from depth-one DFS on $X^1$.
\end{lemma}

\begin{proof}
    Without loss of generality, we assume that $X^1_{:, 1}$ is a foreign key reference to $X^2_{:, 1}$, which is the primary key column of $X^2$. All other columns just have continuous values. In this case, DFS will simply append the corresponding continuous values from $X^2$ to $X^1$.
    Let $j$ be the row number such that $X^2_{j, 1} = X^1_{i, 1}$,
    meaning that $j$ indexes the row of $X^2$ that is referenced by the foreign key in the $i$-th row of $X^1$.  Then, DFS yields
    $$
    \hat{X}^{1,1}_{i,:} = [X^1_{i, 2}, X^1_{i, 3}, \ldots, X^1_{i,C_1}, X^2_{j, 2}, X^2_{j, 3}, \ldots, X^2_{j, C_2}]~.
    $$

    If we set \model{}'s node embedding function $N_\phi(x) = x$ and feature embedding function $E_\theta(x) = x$ both as the identity function, then we can see that
    \begin{gather*}
        h^1(j, 2)=\text{concat}[E_\theta(X^2_{j, :}), h^0(i, 1)]=[X^2_{j, 2}, \ldots, X^2_{j, C_2}, \Vec{0}] \\
        \tilde{h}^1(i, 1) = h^1(j, 2), \tilde{h} \text{ is defined in }Section \ref{sec:message_passing}  \\
        E_\theta(X^1_{i, :}) = [X^1_{i, 2}, X^1_{i, 3}, \ldots, X^1_{i,C_1}]
    \end{gather*}
    Therefore, for all $i$,
    $$
    \begin{aligned}
    h^2(i, 1) &= N_\phi\big(\text{concat}(E_\theta(X^1_{i, :}), \tilde{h}^1(i, 1))\big) \\
    &= [X^1_{i, 2}, X^1_{i, 3}, \ldots, X^1_{i,C_1}, X^2_{j, 2}, X^2_{j, 3}, \ldots, X^2_{j, C_2}, \Vec{0}] \\
    &\supseteq^v \hat{X}^{1,1}_{i,:}
    \end{aligned}
    $$
\end{proof}

We now prove that the same holds for backward relationship between two tables.

\begin{lemma}
    \label{lem:backward}
    Suppose we have two tables $X^1$ and $X^2$ with $C_1$ and $C_2$ columns respectively, and $X^1$ has a backward relationship to $X^2$.  Assume that the columns other than foreign keys and primary keys only have continuous values 
    There exists a two-layer \model{} parametrization such that for any $i$, $h^{2}(i, 1) \supseteq^v \hat{X}^{1,1}_{i,:}=[\texttt{Make\_Features}(X^1, \varnothing, 1)]_{i,:}$, meaning that the second-layer row embedding from \model{} on $X^1$ is a superset vector of the output from depth-one DFS on $X^1$. 
\end{lemma}

\begin{proof}
    Without loss of generality, assume that $X^2_{:, 1}$ is a foreign key reference to $X^1_{:, 1}$, which is the primary key column of $X^1$. All other columns just have continuous values. Let $j_1, j_2, \ldots, j_m \in \{j : X^2_{j, 1} = X^1_{i, 1}\}$ be all the rows in $X_2$ that has foreign key reference to the $i$-th row in $X_1$.  DFS's output in this case becomes:
    $$
    \hat{X}^{1,1}_{i,:} = [X^1_{i, 2}, X^1_{i, 3}, \ldots, X^1_{i,C_1}, \text{Agg}(\hat{M}_{i})]
    $$
    where
    $$
    \begin{gathered}
    \hat{M}_{i}=\{[X^2_{j_1,2},\ldots,X^2_{j_1,C_2}], \ldots, [X^2_{j_m, 2}, \ldots, X^2_{j_m, C_2}]\}\\
    \text{Agg}(\hat{M}_{i}) = [\text{Agg}(\hat{M}_{i})_2, \ldots, \text{Agg}(\hat{M}_{i})_{C_2}] \\
    \text{Agg}(\hat{M}_{i})_k = [\text{Agg}(\{X^2_{j_1, k}, X^2_{j_2, k}, \ldots, X^2_{j_m, k}\})]
    \end{gathered}
    $$
    So function \text{Agg} is doing dimension-wise aggregation for a set of vectors $\hat{M}_{i}$.
    
    Similar to the proof in Lemma~\ref{lem:forward}, we can set $N_\phi(x) = x$, $E_\theta(x) = x$, $M^1(i, 2) = \{h^1(j_1, 2), \ldots, h^1(j_m, 2)\}$ and $\text{Agg}_\psi(\cdot) = \text{Agg}(\cdot)$.  Then    
    $$
    \begin{aligned}
    h^2(i, 1) &= \text{concat}(E_\theta(X^1_{i,:}), \text{Agg}_\psi(M^1(i, 2))) \\
    &\supseteq^v \hat{X}^{1,1}_{i,:}
    \end{aligned}
    $$
\end{proof}

\begin{theorem}
    Suppose we have a database with $N$ tables \\$\{X^1, X^2, \ldots, X^N\}$, each with $C_1, \ldots, C_n$ columns respectively. Assume that the columns other than primary key and foreign key just have continuous values.  Then for any $n \in\{1, \ldots, N\}$ 
    and $l \in \mathbb{N}^+$, there exists a $l+1$-layer \model{} parametrization such that for any $i$, $h^{l+1}(i, n) \supseteq^v \hat{X}^{n,l}=[\texttt{Make\_Features}(X^n, \varnothing, l)]_{i,:}$, meaning that the $l+1$-the layer row embedding from \model{} on $X^n$ is a superset vector of the output of $i$-th row from depth-$l$ DFS on $X^n$.
\end{theorem}

\begin{proof}
    This can be proven by induction. Without loss of generality, we focus on the case where $n=1$.
    
     When $l=1$, only $X^1$ and all the tables that $X^1$ have backward and forward relationships with will matter the result, so this case is a simple extension of Lemma~\ref{lem:forward} and Lemma~\ref{lem:backward} to multiple tables. It is obvious that $h^2(i, 1)\supseteq^v \hat{X}^{1,1}_{i,:}$.
     
     Assume that $h^l(i, 1)\supseteq^v \hat{X}^{1,l-1}_{i,:}$, We need to prove that $h^{l+1}(i, 1)\supseteq^v\hat{X}^{1,l}_{i,:}$.
     
    It is obvious that \begin{align}\label{eq:th4tmp}
        h^l(j, m)\supseteq^v \hat{X}^{m,l-1}_{j,:}~,
    \end{align}
    where $\hat{X}^{m,l-1}=$\texttt{Make\_Features}($X^m, \varnothing, l-1$) for any table $X^m$ and any row number $j$ in this table due to the inductive hypothesis, and then we can follow the process of DFS in Algorithm~\ref{alg:DFS} in line 5-10 of \texttt{Make\_Features}($X^1, \varnothing, l$).

    We denote $X^{'j}\triangleq\texttt{Make\_Features}(X^j,\varnothing, l-1)$, and $F^{'i}$ to be the set that follow the same process of $F^{i}$ in Algorithm~\ref{alg:DFS} but replace \texttt{Rfeat}($X^i$, $X^j$) to \texttt{Rfeat}($X^i$, $X^{'j}$) and same with \texttt{Dfeat}.
    
    If $X^j \notin X_V$, then for any $i_j$
    \begin{gather*}
        X^j=\texttt{Make\_Features}(X^j, \{X^1\}, l-1)~,\\
        \texttt{Rfeat}(X^1, X^j)\subseteq \texttt{Rfeat}(X^1, X^{'j})~.
    \end{gather*}
    If $X^j \in X_V$, then $X^j$ must have been visited in \texttt{Make\_Features}($X^j, X_V, k$) with some $X_V\neq\varnothing$ and $k\leq l-1$. It is obvious that 
    $$
        \texttt{Rfeat}(X^1, X^j)\subseteq \texttt{Rfeat}(X^1, X^{'j})~.
    $$
    We can repeat a similar process in line 11-16, then we can get
    $$
    F^i\subseteq F^{'i}
    $$
    Therefore,  $\hat{X}^{1, l}\subseteq\hat{X}^{'1, l}\triangleq F^{'i}\cup$\texttt{Efeat}($X^{i}$). Actually we do identity mapping for \texttt{Efeat} in our setting.
    \begin{align*}
        h^{l+1}(i,1)=&\text{concat}[E_\theta(X^1_{i,:}), \tilde{h}^{l}(i, f_1), \ldots , \tilde{h}^{l}(i, f_{|F|}),\\
        &\text{Agg}_{\psi_1}(M^{l}(i, b_1)), \ldots,\text{Agg}_{\psi_{|B|}}(M^{l}(i, b_{|B|}))]    
    \end{align*}
    According to \eqref{eq:th4tmp} and the statement we get, we can find that $h^{l}(i_n, n)\supseteq^v X^{'n}_{i_n,:}$($i_n$ is defined in Section \ref{sec:message_passing}) for any $n\in\{f_1,\ldots,f_{|F|}, b_1,\ldots b_{|B|}\}$ and the same for the aggregation part.
    Therefore, follow similar prove in Lemma~\ref{lem:forward} and Lemma~\ref{lem:backward}, we can show that 
    $$
    h^{l+1}(i,1)\supseteq^v\hat{X}^{'1,l}_{i, :} \supseteq^v \hat{X}^{1, l}_{i, :}~.
    $$

\end{proof}

\subsection{Differences between \model{} and RDB2Graph}

The major differences between \model{} and RDB2Graph are as follows: (1) our message functions $N_\phi$ in \eqref{eq:message} and predictor function $\pi_\omega$ \eqref{eq:prediction} are strong differentiable single-table tabular models like DeepFM \cite{guo2017deepfm} and FT-Transformer \cite{gorishniy2021revisiting}, (2) all message functions $N_\phi$ and the downstream predictor function $\pi_\omega$ has skip connections to the raw features of the tables, (3) we use PNA-style aggregations while RDB2Graph only use simple aggregations from e.g. RGCN \cite{schlichtkrull2018modeling} and RGAT \cite{busbridge2019relational}. (4) we use the node embedding function with same parameter for each node types in every layer while RDB2Graph uses functions with different parameter in different layer.  

Differentiable tabular models like DeepFM and FT-Transformer express each column value with a latent embedding vector for single table prediction.  \model{} follows that tradition, \emph{appending the information of each auxiliary table as another latent vector as if there is a new column to the raw features which are regarded as the original columns in the table}, and then feed them all to the downstream predictor.  Such latent vectors of auxiliary tables are in turn computed by another DeepFM or FT-Transformer running on the auxiliary tables themselves, which in turn has more auxiliary tables' information appended.  The end result is that we have multiple DeepFMs/FT-Transformers stacked together, each modeling the feature interaction within its responsible table, and passes the information to other DeepFMs or FT-Transformers.

Moreover, such DeepFMs or FT-Transformers directly takes the raw features from the tables, meaning that each \model{} layer intrinsically has a skip connection to the input features.  Such skip connections are common tricks to mitigate oversmoothing for GNNs (e.g. \cite{chen2020simple}). RDB2Graph in contrast has its predictive capability heavily dependent on the downstream MLP, which is known to be less powerful than DeepFM and FT-Transformer, therefore lacking the ability to model feature interactions within tables and across tables.  They also do not have the skip connections like \model{} does.

Additionally, our choice of using PNA-style aggregations, which consists of \textit{mean}, \textit{min}, \textit{max} at the same time, closely mimic what DFS (specifically Featuretools) does by default: DFS aggregates the same column in multiple different ways, generating new columns for each aggregation. Lastly, to enhance parameter efficiency in our framework, especially considering the use of complex node embedding functions like DeepFM and FT-transformer, we implement a shared parameter approach across each layer.

\section{Experiments}\label{sec:experiments}

\subsection{Synthetic Datasets}
\label{sec:synthetic-experiments}

\begin{figure}
    \centering
    \includegraphics[width=0.8\linewidth]{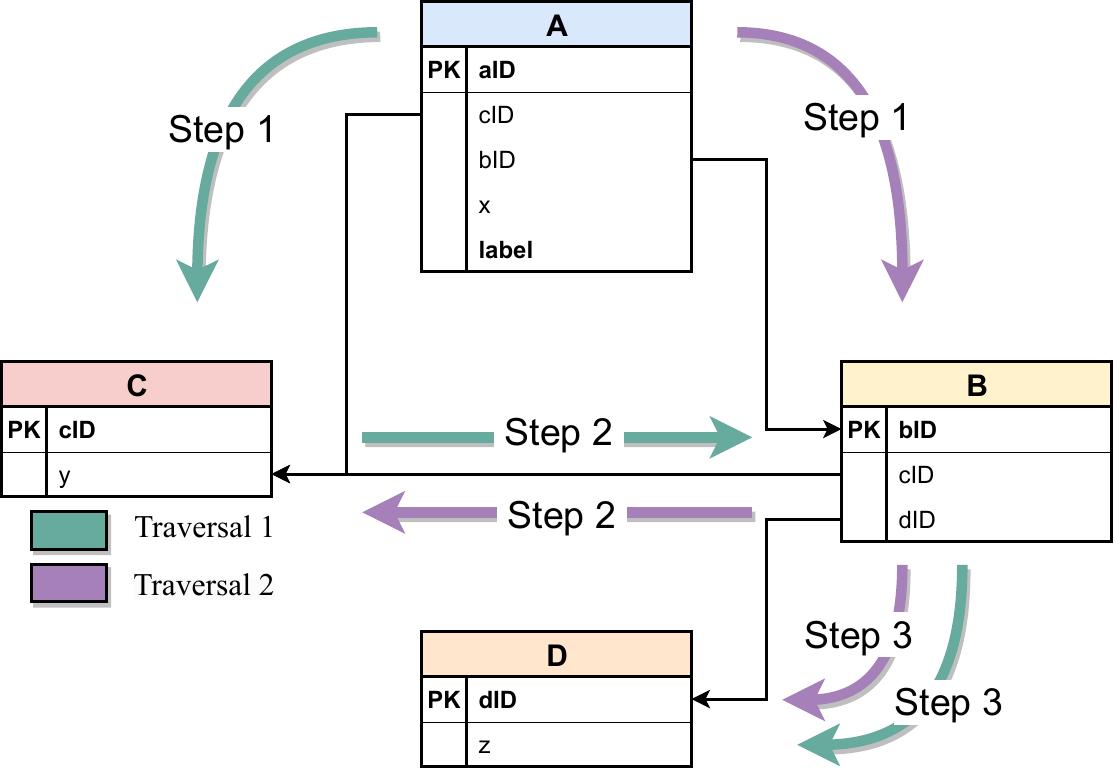}
    \caption{Synthetic Dataset 2's schema (Schema 2) and the traversal path of two traversal types.}
    \label{fig:syn-schema-2}
\end{figure}

We have shown in Section~\ref{sec:property} that DFS is sensitive to traversal order while \model{} is immune to that.  But can \model{} indeed capture the necessary join path among all the paths better than DFS?  We empirically verify \model{}'s capability by actually constructing the two counterexample datasets.
\begin{itemize}[leftmargin=*]
\item \textbf{Schema 1}: The schema corresponding to Fig.~\ref{fig:syn-schema-1} has three tables $A$, $B$ and $C$. Table $A$ is the target table. We set $A.x = 1$ and $C.z = 1$ for all rows in $A$ and $C$, while setting $B.y$ to uniformly random values between [0, 10].  We designate the label column $A.label$ as follows: each row $a$ in $A$ gets a label of 1 iff $\sum_{b: b \in B, b.cID=a.CID} b.y \geq \tau$, and 0 otherwise.  This ensures that the label column in $A$ is only related to the ``golden" join path $B \rightarrow C \rightarrow A$.

\item \textbf{Schema 2}: The schema corresponding to Fig.~\ref{fig:syn-schema-2} has four tables $A, B, C$ and $D$. Table $A$ is the target table. We set $A.x = 1$ and $C.y = 1$ for all rows in $A$ and $C$, while setting $D.z$ to uniformly random values between $[0, 10]$. We designate the label column $A.label$ as follows: each row $a$ in $A$ gets a label of 1 iff $\sum_{b, d: b \in B, d \in D, b.cID=a.CID, b.dID=d.dID} d.z \geq \tau$, and 0 otherwise. This ensures that the label is only related to the ``golden" join path $D\rightarrow B\rightarrow C\rightarrow A$.
\end{itemize}

In our experiments, $\tau$ is set to 25 to ensure that the number of rows with label 1 roughly equals that with label 0. We make both datasets' target tables 1000 rows, and we split train/valid/test sets with the ratio 0.6/0.2/0.2. We set the depth of both DFS and \model{} to 2 for Schema 1 and 3 for Schema 2, as they are the length of golden join paths. The test results are presented in Table \ref{tab:syn}.

\begin{table}
    \caption{AUC results of synthetic datasets.}
    \centering
    \renewcommand\arraystretch{1.25}
    \begin{tabular}{l|l|l}
        \toprule 
        \textbf{Schema1} & depth = 2 & depth = 3 \\
        \midrule
        DFS (Traversal 1) & 0.9983 & - \\
        DFS (Traversal 2) & 0.4686 & 0.9964 \\
        \model{}          & 0.9984 & - \\
        \midrule
        \textbf{Schema 2} & depth = 3 & depth = 4 (max)\\
        \midrule
        DFS (Traversal 1) & 0.9941 & - \\
        DFS (Traversal 2) & 0.9263 & 0.9303 \\
        \model{}          & 0.9986 & - \\
        \bottomrule
    \end{tabular}
    \label{tab:syn}
\end{table}

From these results, we can see that the sensitivity of DFS, as demonstrated in Schema 1 and Schema 2, can indeed result in detrimental effects on the final prediction. In both schemas, Traversal 2 of DFS fails to capture the golden join path, resulting in limited learning capability and a significant drop in performance. In contrast, \model{}'s ability to cover all possible join paths proves advantageous in both Schema 1 and Schema 2 when the number of search depth is right.

One might wonder if increasing the depth of DFS could eventually allow it to discover the join paths ignored in smaller depths.  To verify that, we also tried increasing the depth until either DFS stops augmenting with more features, or reaches 0.99 AUC result, which we consider as successfully capturing the golden join path.  We can see that DFS failed in Schema 2 even if we increase the depth to the maximum of 4, beyond which it stops augmenting.  Even if increasing the depth did help in Schema 1, it will cost much more time and memory due to the exponential growth of augmented features.  These observations further underscore the need for alternative approaches, such as \model{}, which is immune to the traversal order sensitivity like DFS. %

\subsection{Real-World Datasets Experiment Setup}
\subsubsection{Datasets Discription}
\begin{table*}
    \caption{Dataset statistics.}
    \renewcommand\arraystretch{1.25}
    \centering
    \begin{tabular}{l|c|c|c|c|c|c|c}
        \toprule
        Dataset & \# Train & \# Valid & \# Test & \# Tables & \# Total Rows & \# Foreign Keys & \# Total Columns \\
        \midrule
        AVS         & 96035  & 32011 & 32011 & 3 & 8.2M & 2  & 23  \\
        Outbrain    & 828251 & 21796 & 21796 & 9 & 28M  & 10 & 23  \\
        Diginetica  & 882415 & 20521 & 20521 & 6 & 2.7M & 7  & 20  \\
        KDD15       & 72325  & 24108 & 24108 & 4 & 8.3M & 3  & 19  \\
        \bottomrule
    \end{tabular}
    \label{Dataset Information}
\end{table*}

We evaluate the effectiveness of our proposed model on four real-world datasets: 
\textit{Acquire-valued-shoppers, Outbrain, Diginetica, KDD2015}, They span across different domains, including customer retention prediction, click-through rate prediction, recommendation, fraud detection, etc.
\begin{itemize}[leftmargin=*]
\item \textbf{AVS (acquire-valued-shoppers)} \cite{acquire-valued-shoppers-challenge}: The Acquire-valued-shoppers dataset is to predict which shoppers are most likely to repeat purchases. It has provided complete, basket-level, pre-offer shopping transaction history for a large set of shoppers who were targeted for an acquisition campaign. The incentive offered to that shopper and their post-incentive behavior is also provided.
\item \textbf{Outbrain} \cite{outbrain-click-prediction}: Outbrain is a platform that pairs relevant advertisements to curious readers on the web pages they would visit. In this dataset, the goal is to predict which advertisement a user will click on, given that the user is browsing a web page (called a \emph{display}). The user behaviors and the details of the advertisements are also provided. %
\item \textbf{Diginetica} \cite{diginetica}: This dataset is provided by DIGINETICA and its partners containing anonymized search and browsing logs, product data, and anonymized transactions. The goal is to predict which products will be clicked among all the return results of one search query according to the personal shopping, search, and browsing preferences of the users. 
\item \textbf{KDD15} \cite{KDD15}: This dataset is collected by XuetangX, one of the largest MOOC platforms in China. The goal of this dataset is to predict whether a user will drop a course within the next 10 days based on his or her prior activities.
\end{itemize}

\subsubsection{Data Preprocessing}

For all datasets, we preprocess datetime columns to four categorical columns: year, month, day, and day of week.  During prediction, we also drop the primary key and foreign key columns to ensure that the model does not memorize the labels based on IDs.  The statistics of the resulting datasets are shown in Table~\ref{Dataset Information}. The preprocessing code will be released together with our model implementation.

\subsubsection{Evaluation Metrics}
To quantitatively evaluate the model performances, we use the area under ROC curve (\textit{AUC}), which is widely used for binary classification tasks.

\subsubsection{Baselines}
We compare our method against two categories of approaches.

The first category involves offline methods that operate on relational databases by consolidating the data into a single table, followed by the application of single-table models. Other than the most important open-source method DFS, we compared our method against the following:
\begin{itemize}[leftmargin=*]
    \item \textbf{Target-table Only (TT)}.  We only use the target table for prediction.
    \item \textbf{Simple join (SJ)}.  If table $X^A$ has a forward relationship to table $X^B$, it is possible to join the two tables by simply appending the columns from $X^B$ to $X^A$. For instance, in Fig. \ref{fig:AVS}, the History table has a forward relationship to the Offer table, so we can join the two tables by appending the columns from the Offer table to the History table. This merging process is equivalent to performing a left merge on the Offer table with the History table based on the offer ID. This can be extended to recursive joining of further forward relationships with multiple levels. For example, if target table $X^A$ has a forward relationship to table $X^B$, and table $X^B$ has a forward relationship to table $X^C$, then we can join table $X^C$ to $X^B$ and subsequently join table $X^B$ to $X^A$. This recursive join method gathers information from as many tables as possible without aggregation.  We refer to this as simple join.
\end{itemize}

The second category involves RDB2Graph, which embeds each row into a node embedding vector and applies Graph Neural Network (GNN) models directly on the top. We evaluated our method against RDB2Graph using popular GNN models as the backbone, including RGCN \cite{schlichtkrull2018modeling}, Relational GAT \cite{busbridge2019relational}, and HGT \cite{hu2020heterogeneous}, the last considered state-of-the-art for heterogeneous graphs.

Although in Section~\ref{sec:related_works} we mentioned 
more methods such as OneBM\cite{lam2017one}, R2N\cite{lam2018neural}, ARDA\cite{arda}, AutoFeature\cite{liu2022feature}, METAM \cite{galhotra2023metam} and ATJ-Net\cite{bai2021atj},  they are either proprietary or not having their source code released, so we cannot compare with them in the experiments.

In both \model{} and the first category of methods, we evaluate DeepFM \cite{guo2017deepfm} and FT-transformer \cite{gorishniy2021revisiting} as the base model and node embedding function choices 
Since our framework can plug in any differentiable model designed for single table settings, we just use these two powerful and representative single table models as evaluation, and any other differentiable mentioned in Section~\ref{sec:related_works} can be simply merged into our framework in practical.

\subsubsection{Parameter Settings}

\begin{table*}[ht]
    \caption{AUC results of real-world datasets.
    }
    \centering
    \renewcommand\arraystretch{1.25}
    \begin{tabular}{l|c|c|c|c}
        \toprule
        Model & \textbf{AVS} & \textbf{Outbrain} & \textbf{Diginetica} & \textbf{KDD15} \\
        \midrule
        Search Depth or \#GNN layers & 2 & 4 & 3 & 2 \\
        \midrule
        TT + DeepFM & 0.6737 $\pm$ 0.0003 & -                   & -                   & -            \\
        TT + FT-Transformer   & 0.6720 $\pm$ 0.0007 & -                   & -                   & -  \\
        SJ + DeepFM & 0.6902 $\pm$ 0.0003 & 0.7223 $\pm$ 0.0008 & 0.6278 $\pm$ 0.0074 & 0.6297 $\pm$ 0.0028 \\
        SJ + FT-Transformer   & 0.6894 $\pm$ 0.0005 & 0.7188 $\pm$ 0.0014 & 0.6100 $\pm$ 0.0089 & 0.6061 $\pm$ 0.0065\\
        \midrule
        RDB2Graph + RGCN & 0.6956 $\pm$ 0.0007 & 0.7420 $\pm$ 0.0002 & 0.7420 $\pm$ 0.0184 & 0.8557 $\pm$ 0.0025\\
        RDB2Graph + GAT  & 0.6978 $\pm$ 0.0006 & 0.7440 $\pm$ 0.0003 & 0.7565 $\pm$ 0.0189 & 0.8629 $\pm$ 0.0026\\
        RDB2Graph + HGT  & 0.6957 $\pm$ 0.0006 & 0.7549 $\pm$ 0.0002 \textsuperscript{*} & 0.8070 $\pm$ 0.0005 & 0.8719 $\pm$ 0.0010\\
        \midrule
        DFS + DeepFM & 0.6974 $\pm$ 0.0006 & 0.7337 $\pm$ 0.0008 & 0.7963 $\pm$ 0.0055 & 0.8717 $\pm$ 0.0011\\
        DFS + FT-Transformer   & 0.6916 $\pm$ 0.0007 & 0.7303 $\pm$ 0.0012 & 0.8024 $\pm$ 0.0018 & 0.8626 $\pm$ 0.0022\\
        \midrule
        \model{} (ours) & \textbf{0.7001} $\pm$ 0.0002 & \textbf{0.7556} $\pm$ 0.0008 & \textbf{0.8106} $\pm$ 0.0023 & \textbf{0.8781} $\pm$ 0.0011\\
        \bottomrule
        \multicolumn{5}{l}{\textsuperscript{*} RDB2Graph + HGT can only run up to 3 layers due to GPU memory constraints.} %
    \end{tabular}
    \label{tab:real_dataset}
\end{table*}

\begin{table}[ht]
    \caption{Training time (in hours) of different models. The result marked as * is run with search depth = 3 due to OOM problem; however, the other results in the same column are run using search depth = 4. And the results from \model{} are \model{} + DeepFM.}
    \centering
    \renewcommand\arraystretch{1.25}
    \begin{tabular}{l|c|c|c|c}
        \toprule %
        \multirow{1}{*}{Model} & \textbf{AVS} & \textbf{Outbrain} & \textbf{Diginetica} & \textbf{KDD15}\\
        \midrule
        RGCN & 0.25 & 4.36 & 2.37 & 0.56 \\
        GAT  & 0.33 & 6.00 & 1.76 & 0.51 \\
        HGT  & 1.17 & *8.31 & 4.83 & 2.45 \\
        \model{} & 0.60 & 8.60 & 3.16 & 0.39 \\
        \bottomrule
    \end{tabular}
    \label{tab:time_compare}
\end{table}

In our experiments, we use Weight \& Bias \cite{wandb} to automatically search for the best hyperparameter combination. %
We run 50 hyperparameter trials for all our experiments, with distribution $LogUniform(10^{-7}, 10^{-3})$ for learning rate and weight decay. For models that use FT-Transformer as the base model, we set the distribution for dropout probability to $Uniform(0, 0.3)$, while for other models, we set it to $Uniform(0, 0.2)$.  We set the hidden dimension $d$ to 16, and the node embedding size of RDB2Graph to 64.  After hyperparameter optimization, we select the one with the best performance on the validation set, and then we rerun the model five times using this hyperparameter combination to reduce the impact of randomness. We record the mean performance and the standard deviation on the test set based on these reruns. We set the search depth of DFS and \model{} as well as the number of layers in RDB2Graph to the largest value possible to fit computation into GPU memory.

\subsection{Performance Comparison}
\label{sec:performance-comparison}
\begin{table*}[ht]
    \caption{AUC results of \model{} with different row embedding functions.}
    \centering
    \renewcommand\arraystretch{1.25}
    \begin{tabular}{l|c|c|c|c}
        \toprule
        Model & \textbf{AVS} & \textbf{Outbrain} & \textbf{Diginetica} & \textbf{KDD15}\\
        \midrule
        \model{} ($N_\phi=\text{MLP},\pi_\omega=\text{DeepFM}$) & 0.6949 $\pm$ 0.0026 & 0.7540 $\pm$ 0.0005 & 0.8095 $\pm$ 0.0024 & 0.8660 $\pm$ 0.0029\\
        \model{} ($N_\phi=\text{DeepFM},\pi_\omega=\text{DeepFM}$)     & \textbf{0.7001} $\pm$ 0.0002 & \textbf{0.7556} $\pm$ 0.0008 & \textbf{0.8106} $\pm$ 0.0023 & \textbf{0.8781} $\pm$ 0.0011\\
        \bottomrule
    \end{tabular}
    \label{tab:row_embedding_ablation}
\end{table*}

\begin{table*}[ht]
    \caption{AUC results of \model{} with different base model.}
    \centering
    \renewcommand\arraystretch{1.25}
    \begin{tabular}{l|c|c|c|c}
        \toprule
        Model & \textbf{AVS} & \textbf{Outbrain} & \textbf{Diginetica} & \textbf{KDD15}\\
        \midrule
        \model{} with DeepFM & \textbf{0.7001} $\pm$ 0.0002 & \textbf{0.7556} $\pm$ 0.0008 & \textbf{0.8106} $\pm$ 0.0023 & \textbf{0.8781} $\pm$ 0.0011\\
        \model{} with FT-Transformer & 0.6936 $\pm$ 0.0010 & 0.7543 $\pm$ 0.0006 & 0.8082 $\pm$ 0.0012 & 0.8703 $\pm$ 0.0011 \\
        \bottomrule
    \end{tabular}
    \label{tab:base_model_ablation}
\end{table*}

The results are in Table~\ref{tab:real_dataset}.  We note that
TT results are not applicable on most datasets, as the target table in Outbrain, Diginetica and KDD15 contains only foreign keys and the label column; it is therefore meaningless for single-table models since the foreign key columns are dropped to avoid memorization of labels.

We also compare the training time for each model and all the experiments are run on g4dn.metal. We focused on comparing the \model{} framework with other GNN models, as these models employ online sampling, which involves aggregating information from other tables during the training process. On the other hand, baselines such as DFS and SJ adopt offline sampling, where information from other tables is aggregated to form a single target table prior to model training, so these methods cannot be fairly compared. The training time results, specifically for \model{} combined with DeepFM, are presented in Table \ref{tab:time_compare}.

From the real-world dataset results, we have the following observation and analysis.
\begin{itemize}[leftmargin=*]
    \item \model{} consistently outperforms other baselines or achieves similar results to the best models across all four datasets. In contrast, other models may fail to perform well in certain datasets. Specifically, HGT performs well in Outbrain but fails in AVS. DFS, RGCN and GAT do not perform well across all datasets.
    \item While HGT may appear to be the most powerful baseline, it is important to consider the training time when comparing with \model{}. When comparing the training time of HGT and \model{}, we find that \model{} combined with DeepFM is more efficient than HGT throughout the entire training process. In fact, \model{} requires only half to one-third of the training time per epoch compared to HGT in most datasets. The similar training time observed in the Outbrain dataset is due to the fact that HGT utilizes a search depth of 3, while other models employ a search depth of 4.
\end{itemize}

\subsection{Ablation Study}

\begin{figure*}
    \centering
    \includegraphics[width=\linewidth]{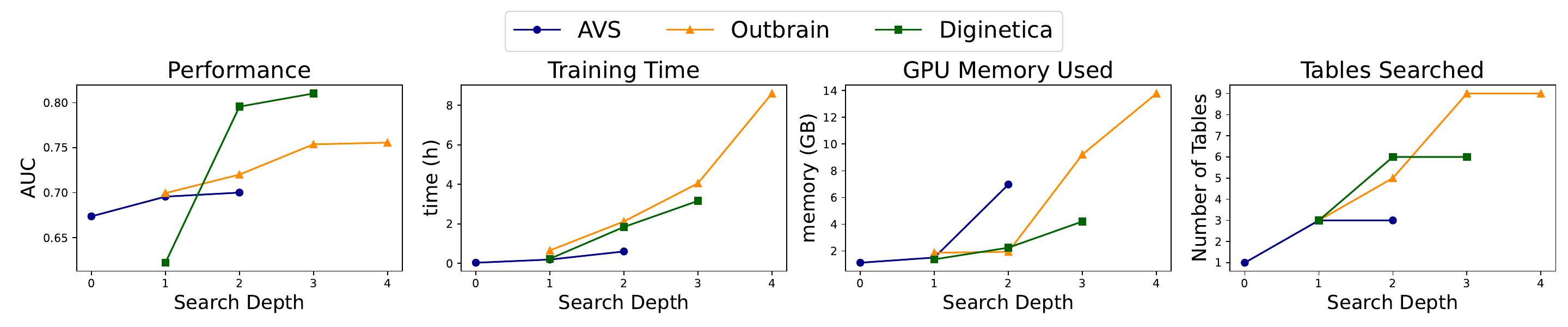}
    \caption{The performance, training time, GPU memory used, and number of tables searched of \model{} at different search depths in three different datasets.}
    \label{fig:search depth}
\end{figure*}

To gain a deeper understanding of the design rationale behind our proposed framework, \model{}, we conducted several ablation experiments.
\subsubsection{Row Embedding Function}
In our experiments we use the same architecture for base model $\pi_\omega$ and row embedding function $N_\phi$.  While the choice of base model $\pi_\omega$ is straightforward as \model{} can enjoy the capacity of a strong single-table model, using DeepFM and FT-Transformer for $N_\phi$ seems redundant at first glance.  So we compared our model against an alternative where we use an MLP for $N_\phi$ while still using DeepFM as base model.  The results are in Table~\ref{tab:row_embedding_ablation}.

We can have several observations and analyses from the results.
\begin{itemize}[leftmargin=*]
    \item The importance of a powerful row embedding function becomes evident in the AVS and KDD15 datasets, as there is a significant drop in performance when we replace the previous row embedding function with a simple MLP.
    This indicates that the feature interactions from DeepFM and FT-Transformer are necessary for the tables other than the target table as well.
    \item In the Outbrain and Diginetica datasets, the row embedding function appears to have less impact on performance. This observation can be attributed to the fact that the prediction of the target column relies more on the relationships between rows in different tables rather than the interaction of row features. This hypothesis aligns well with the results obtained from GNN models, where MLP is used as the embedding function to form node embedding vectors. Notably, the performance of the most powerful graph model, HGT, is similar to that of \model{}, supporting the hypothesis.
\end{itemize}

\subsubsection{Base Model}
We experimented with different base models for \model{}, and it is essential to note that the row embedding function is also adjusted accordingly based on the chosen base model. The results can be seen in Table \ref{tab:base_model_ablation}, where we find that DeepFM demonstrates superior performance across all the datasets.  This is perhaps unsurprising given that DFS+DeepFM is better than DFS+FT-Transformer on 3 of 4 datasets in Table \ref{tab:real_dataset}.  We also remark that while DeepFM may perform best for now, any strong differentiable tabular model introduced in the future could potentially serve as an effective replacement. 

\subsection{Effect of Search Depth}
Search depth is an important architectural parameter of \model{} when applied to real-world datasets. In Fig. \ref{fig:search depth}, we present the \textit{Performance}, \textit{Training Time}, \textit{GPU Memory Used}, and \textit{Number of Tables Searched} by \model{} at different search depths in the \textit{AVS}, \textit{Outbrain}, and \textit{Diginetica} datasets for larger search depth range. We gradually increase the search depth until we exhaust the GPU memory. We do not report the results for \textit{Outbrain} and \textit{Diginetica} at search depth 0 as it is meaningless, similar to why TT is meaningless as explained in Section~\ref{sec:performance-comparison}. It is observed that the Training Time and GPU Memory Used tend to increase rapidly, while the performance of \model{} may reach a bottleneck. Therefore, in real-world applications of \model{}, selecting an appropriate search depth becomes an important tradeoff between performance and resource utilization (i.e., training time and GPU memory used).

\section{Conclusion}\label{sec:conclusion}
In this paper, we introduce \model{}, a novel framework specifically designed for general \emph{column prediction} tasks on relational databases. \model{} is an embedding update and prediction framework that can plug in any differentiable model designed for single table settings as base model or embedding function to finalize the specific model. \model{} also improves upon the previous methods DFS and RDB2Graph and and tackles the inherent problems in previous methods. We conduct comprehensive experiments on real-world datasets to evaluate the performance of \model{} against other baseline methods. The results demonstrate that \model{} consistently outperforms the baselines and exhibits superior efficiency compared to the most powerful and complex GNN baseline, HGT. The successful implementation and evaluation of \model{} highlight its potential as an effective and efficient solution for machine learning tasks on relational databases. 

\begin{acks}
We thank the support from Shanghai Municipal Science and Technology Major Project (2021SHZDZX0102) and National Natural Science Foundation of China (62177033), as well as computation resource support from AWS. 
\end{acks}

\bibliographystyle{ACM-Reference-Format}
\bibliography{reference}

\end{document}